\documentclass[letterpaper, 10 pt, conference]{ieeeconf}  
\IEEEoverridecommandlockouts                              
\usepackage{graphicx}
\usepackage{epsfig} 
\usepackage{mathptmx}
\usepackage{mathtools}
\usepackage{amsmath} 
\usepackage{bm}

\usepackage{amssymb}  
\usepackage{amsthm}
\usepackage{siunitx}
\usepackage{wasysym}
\usepackage[mathcal]{euscript}
\usepackage{bbm}
\usepackage{color}
\usepackage{lipsum}
\usepackage[ruled,vlined,linesnumbered]{algorithm2e}
\usepackage[colorlinks=true,linkcolor=black,anchorcolor=black,citecolor=black,filecolor=black,menucolor=black,runcolor=black,urlcolor=black]{hyperref}
\usepackage{etoolbox}
\usepackage[table,xcdraw,dvipsnames]{xcolor}
\usepackage{multirow}
\usepackage{bbm}
\usepackage{outlines}
\let\labelindent\relax
\usepackage{enumitem}
\setenumerate[2]{label=\alph*.}
\setenumerate[3]{label=\roman*.}

\makeatletter
\patchcmd{\@makecaption}
  {\scshape}
  {}
  {}
  {}
\makeatother

\usepackage[
    style=ieee,
    doi=false,
    isbn=false,
    url=false,
    eprint=false,
    backend=biber,
    natbib=true
    ]{biblatex}

\usepackage{caption}
\usepackage{subcaption}
\usepackage{marginnote}
\bibliography{references}
\newcounter{change}

\newtheorem{defn}{Definition}

\newtheorem{lem}[defn]{Lemma}
\newtheorem{prop}[defn]{Proposition}
\newtheorem{assum}[defn]{Assumption}

\newtheorem{thm}[defn]{Theorem}

\newtheorem{problem}[defn]{Problem}

\providecommand{\R}{\ensuremath \mathbb{R}}
\providecommand{\N}{\ensuremath \mathbb{N}}

\newcommand{\Xplan}{P}
\newcommand{\Xgoi}{Q}

\newcommand{\K}{K}
\newcommand{\D}{S} 
\newcommand{\Dfun}{\ensuremath \mathbb{S}} 
\newcommand{\Xgoal}{G}

\newcommand{\Eset}{E}

\newcommand{\Rbrs}{\Omega}

\newcommand{\Obs}{\mathcal{O}}

\newcommand{\Poly}{V}
\newcommand{\Qoly}{U}

\newcommand{\Ravd}{\Lambda}

\newcommand{\Obsbrs}{B}

\newcommand{\regtext}[1]{\mathrm{\textnormal{#1}}}

\newcommand{\vc}[1]{\mathbf{#1}}
\newcommand{\lbl}[1]{_{\regtext{#1}}}
\newcommand{\itv}[1]{\tilde{#1}}

\newcommand{\zeros}{\mathbf{0}}

\newcommand{\eye}{\mathbf{I}}

\newcommand{\proj}[2][]{\regtext{proj}_{#1}\!\left(#2\right)}
\newcommand{\projtext}[1]{\regtext{proj}_{#1}}

\newcommand{\abs}[1]{\left\vert#1\right\vert}

\newcommand{\union}{\bigcup}

\newcommand{\tp}{^\intercal} 

\newcommand{\dyn}{\vc{f}}

\newcommand{\final}{\lbl{f}}

\newcommand{\hpoly}{\mathcal{H}}
\newcommand{\ahpoly}{\mathcal{AH}}
\newcommand{\conv}[1]{\regtext{conv}\!\left(#1\right)}
\newcommand{\rpm}[1]{\regtext{RPM}\!\left(#1\right)}

\newcommand{\brsfunction}{\mathcal{B}}
\newcommand{\brs}[2][]{{\brsfunction}_{#1}\!\left(#2\right)}

\newcommand{\pwafunction}{\psi}

\newcommand{\pwa}[2][]{\vc{\pwafunction}_{#1}\!\left(#2\right)}

\newcommand{\nn}{\vc{\xi}}
\newcommand{\inset}{X}

\newcommand{\depth}{d}
\newcommand{\idxo}{\idx{0}}

\newcommand{\Wt}{\vc{W}}
\newcommand{\bias}{\vc{w}}

\newcommand{\ts}{t}

\newcommand{\is}{i}
\newcommand{\js}{j}

\newcommand{\ps}{p}
\newcommand{\ds}{d}

\newcommand{\gams}{\gamma}
\newcommand{\sap}{s}
\newcommand{\tf}{\ts\final}
\newcommand{\es}{e}

\newcommand{\xv}{\vc{x}}
\newcommand{\yv}{\vc{y}}

\newcommand{\dv}{\vc{s}} 

\newcommand{\planv}{\vc{p}}
\newcommand{\goiv}{\vc{q}}
\newcommand{\paramv}{\vc{k}}

\newcommand{\Acon}{\vc{A}}
\newcommand{\bcon}{\vc{b}}
\newcommand{\Ccon}{\vc{C}}
\newcommand{\dcon}{\vc{d}}

\newcommand{\px}{p_x}
\newcommand{\py}{p_y}

\newcommand{\pxg}{p_{x, g}}
\newcommand{\pyg}{p_{y, g}}
\newcommand{\pydes}{y_{\regtext{des}}}
\newcommand{\vdes}{v_{\regtext{des}}}
\newcommand{\pth}{\theta}

\newcommand{\vel}{v}
\newcommand{\sideslip}{\beta}

\newcommand{\ndim}{{n}}
\newcommand{\mdim}{{m}}
\newcommand{\pdim}{{p}}
\newcommand{\qdim}{{q}}

\newcommand{\nparam}{\ndim_k}

\newcommand{\nplan}{\ndim_p}
\newcommand{\npwa}{\ndim\lbl{PWA}}

\newcommand{\nhp}{\ndim\lbl{h}}
\newcommand{\nobs}{\ndim_{\Obs}}

\newcommand{\idx}[1]{_{#1}} 

\newcommand{\sample}{\regtext{sample}}

\title{\LARGE \bf
Guaranteed Reach-Avoid for Black-Box Systems through \\ Narrow Gaps via Neural Network Reachability
}
\author{Long Kiu Chung$^{1}$,
Wonsuhk Jung$^{1}$,
Srivatsank Pullabhotla$^{1}$,
Parth Shinde$^{1}$,
Yadu Sunil$^{1}$,\\
Saihari Kota$^{1}$,
Luis Felipe Wolf Batista$^{2}$,
C\'edric Pradalier$^{2}$,
and Shreyas Kousik$^{1}$
\thanks{
$^{1}$Georgia Institute of Technology, Atlanta, GA.
$^{2}$Georgia Tech Europe, Metz, France.
This work was supported by the Georgia Tech AIMPF.
\textbf{Corresponding author: }\texttt{lchung33@gatech.edu}.
\textbf{Website: }\href{https://saferoboticslab.me.gatech.edu/research/neuralparc/}{https://saferoboticslab.me.gatech.edu/research/neuralparc/}.
\textbf{GitHub: }\href{https://github.com/safe-robotics-lab-gt/NeuralPARC}{https://github.com/safe-robotics-lab-gt/NeuralPARC}.
}}
\begin{document}

\maketitle

\begin{abstract}
In the classical \textit{reach-avoid} problem, autonomous mobile robots are tasked to \textit{reach} a goal while \textit{avoiding} obstacles.
However, it is difficult to provide guarantees on the robot's performance when the obstacles form a \textit{narrow gap} and the robot is a \textit{black-box} (i.e. the dynamics are not known analytically, but interacting with the system is cheap).
To address this challenge, this paper presents NeuralPARC.
The method extends the authors' prior Piecewise Affine Reach-avoid Computation (PARC) method to systems modeled by rectified linear unit (ReLU) neural networks, which are trained to represent parameterized trajectory data demonstrated by the robot.
NeuralPARC computes the reachable set of the network while accounting for modeling error, and returns a set of states and parameters with which the black-box system is guaranteed to reach the goal and avoid obstacles.
NeuralPARC is shown to outperform PARC, generating provably-safe extreme vehicle drift parking maneuvers in simulations and in real life on a model car, as well as enabling safety on an autonomous surface vehicle (ASV) subjected to large disturbances and controlled by a deep reinforcement learning (RL) policy.

\end{abstract}

\section{Introduction} \label{sec:intro}

Many important mobile robot planning and control problems involve systems that are difficult or even impossible to model analytically \cite{nguyen2011model}.
That said, interactions with a system may be cheap, such that a dataset of the systems' input, control, and output signals is readily available---i.e., a \textit{black-box} system.
Data availability makes such systems inviting for learning-based control, but it is challenging to certify that learning-based motion planning or control will operate safely \cite{liu2021algorithms}.
Of course, ensuring collision avoidance is straightforward if a robot moves overly cautiously or remains stopped, so we also desire \textit{liveness}.

We define safety and liveness through the common framework of a \textit{reach-avoid} problem, where an agent must navigate to a set of goal states without colliding with obstacles.
In particular, we are interested in the case where obstacles form \textit{narrow gaps}; in this setting, it is especially difficult to maintain guarantees because solutions should be conservative to account for errors in estimation and performance, but cannot be overly conservative such that solutions cannot be found \cite{chung2024goal}.
In this paper, we present a method that returns a set of safe initial states and parameters, with which a black-box, autonomous mobile robot is guaranteed to \textit{reach} and \textit{avoid} in narrow-gap scenarios.
An overview and two examples of our approach are shown in Fig.~\ref{fig:front_figure}.

\begin{figure}[t]
\centering
    \includegraphics[width=1\columnwidth]{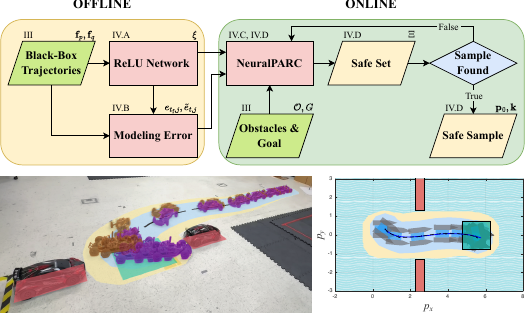}
\caption{
(Top) A flowchart of our Neural Piecewise Affine Reach-avoid Computation (NeuralPARC) method labelled with relevant paper sections and symbols.
We test NeuralPARC on (bottom left) extreme vehicle drift parking and (bottom right) an autonomous surface vehicle (ASV) controlled by deep reinforcement learning (RL) and subject to large disturbances.
In each example, a timelapse of the realized motion of the agents are shown.
The blue tube is the overapproximation of the agent's body as a circle swept across NeuralPARC's predicted trajectory (dashed line), and the yellow tube represents NeuralPARC's modeling error bounds.
We denote the actual trajectory of the ASV with a solid line, and show two timelapses (orange and purple) of the drifting vehicle following the \textit{same} motion plan.
Despite the large variance in the robots' tracking performance, NeuralPARC \textit{always} guarantees the agents to reach the green goal and avoid the red obstacles.
}\label{fig:front_figure}
\vspace*{-0.5cm}
\end{figure}

\subsection{Related Work}


To solve reach-avoid problems for black-box systems, a popular direction is to train \textit{neural networks} to model the unknown black-box dynamics using the observation data, then verify the networks using reachability analysis \cite{tran2019star, tran2020nnv, tran2020verification, chung2021constrained, althoff2010reachability, everett2020robustness,liu2021algorithms}.
However, these methods typically require the designer to propose new policies through trial-and-error if reach-avoid guarantees cannot be verified on the tested policy.

Safe RL techniques address these challenges by creating a reward system that incentivizes reaching the goal while penalizing safety violations throughout the learning and deployment phases \cite{achiam2017constrained, schulman2015trust, schulman2017proximal,selim2022safe}.
However, pure safe RL methods can have poor safety and liveness rates in practice \cite{qin2022sablas}.

A recent branch of work has been focused on learning certificate functions such as Control Lyapunov Functions (CLF) and Control Barrier Functions (CBF) using neural networks \cite{qin2022sablas, dai2021lyapunov, zhang2023data}.
Though certificate functions are traditionally successful in systems with known dynamics \cite{ames2019control, freeman2008robust, choi2021robust, qin2021learning}, their guarantees are not maintained for black-box agents due to the inherent approximation errors in deep learning.

Finally, while planning a black-box agent through narrow gaps is an active area of research in sampling-based motion planning \cite{elbanhawi2014sampling, orthey2024multilevel, liu2023simultaneous, yu2023gaussian}, challenges arise when the dynamics are too nonlinear, the system is too high-dimensional, or when the gap is too tight, such that sampling from the small set of feasible solution while maintaining kinodynamic feasibility is too difficult \cite{wu2020r3t, berenson2009manipulation}.
As such, these methods often have long computational time or are very conservative in approximating the robot's motion.

Our prior work on Piecewise Affine Reach-avoid Computation (PARC) addressed the reach-avoid problem in narrow-gap scenarios \cite{chung2024goal}.
PARC was shown to outperform sampling-based motion planning, certificate function, and other reachability methods by achieving low conservativeness using piecewise affine (PWA) systems and H-polytopes to model trajectories and reachable sets. 
However, in PARC, the construction of the trajectory model requires intuitive knowledge about the system dynamics.
Moreover, PARC requires a nominal, goal-reaching trajectory to begin analysis, and is only capable of analyzing reach-avoid guarantees with respect to the subset of the PWA system corresponding to this nominal plan.
As such, we propose NeuralPARC, an extension to PARC that retains its advantages but not its weaknesses by using neural networks for modeling and verification.

\subsection{Contributions}

In this paper, we propose NeuralPARC, which improves upon PARC in three ways:

\begin{enumerate}[wide]
    \item Instead of carefully hand-crafting a trajectory model, NeuralPARC learns a trajectory model with rectified linear unit (ReLU) neural networks from system trajectory data.
    This data-driven approach enables NeuralPARC to operate on \textit{black-box} models, without requiring intuitive knowledge about the system dynamics to uphold performance.

    \item By exploiting Reachable Polyhedral Marching (RPM) \cite{vincent2021reachable}, NeuralPARC is certified to eventually explore \textit{all} affine dynamics generatable by the ReLU network trajectory model, whereas PARC can only analyze \textit{one} affine dynamical system given a nominal goal-reaching plan.
    Thus, NeuralPARC do not require a nominal plan to begin analysis, and can generate more diverse safe trajectories than PARC.

    \item As ``universal approximators'' \cite{hornik1989multilayer}, neural networks enable NeuralPARC to attain lower modelling errors and tighter approximations than PARC.
    We demonstrate this in simulations and on hardware with extreme vehicle drift parking maneuvers and an autonomous surface vehicle (ASV) agent under large disturbances and controlled by deep RL.
\end{enumerate}
\section{Preliminaries} \label{sec:preliminary}

We now introduce our notation for H-polytopes, AH-polytopes, PWA systems, and ReLU neural networks.

\subsection{Set Representations}

In this work, we represent most sets as either H-polytopes or AH-polytopes, which have extensive algorithm and toolbox support \cite{herceg2013multi, althoff2015introduction, sadraddini2019linear}. 

\subsubsection{H-Polytopes}

An $\ndim$-dimensional H-polytope $\hpoly(\Acon, \bcon) \subset \R^{\ndim}$ is a closed convex set parameterized by $\nhp$ linear constraints $\Acon \in \R^{\nhp \times \ndim}$ and $\bcon \in \R^{\nhp}$ as $\hpoly(\Acon, \bcon) = \left\{\xv\ |\ \Acon\xv\leq\bcon\right\}$.
Its emptiness $\hpoly(\Acon, \bcon) = \emptyset$ can be checked with a single linear program (LP).
We use their closed-form representations in intersections $\cap$ and Cartesian products $\times$, and compute their Minkowski sum $\oplus$ by projection and their Pontryagin difference by solving an LP for each constraint in the subtracted polytope \cite{herceg2013multi}.
We note that closed-form expressions of Minkowski sum and Pontryagin difference exist if all H-polytopes involved are hyperrectangles \cite{forets2021lazysets}.

\subsubsection{AH-Polytopes}

An $\mdim$-dimensional AH-polytope $\ahpoly(\Acon, \bcon, \Ccon, \dcon) \subset \R^{\mdim}$ is a closed, convex set parameterized by an $\ndim$-dimensional H-polytope $\hpoly(\Acon, \bcon)$ and an affine map defined by $\Ccon \in \R^{\mdim \times \ndim}$ and $\dcon \in \R^{\mdim}$ as 
\begin{align}
    \ahpoly(\Acon, \bcon, \Ccon, \dcon) = \{\Ccon\xv + \dcon\ |\ \xv \in \hpoly(\Acon, \bcon)\}.
\end{align}
AH-polytopes enable closed-form intersections $\cap$ and convex hulls $\conv{\cdot}$.
The projection of an $\ndim$-dimensional H-polytope into its first $\mdim$ dimensions $\proj[\mdim]{\cdot}$ also has a closed-form expression as an AH-polytope.
Finally, both the point containment $\yv \in \ahpoly(\Acon, \bcon, \Ccon, \dcon)$ and emptiness $\ahpoly(\Acon, \bcon, \Ccon, \dcon) = \emptyset$ of an AH-polytope can each be checked with one LP \cite{sadraddini2019linear}.

\subsection{PWA Systems}

A PWA system is a continuous function $\pwafunction: \inset \to \R^{\mdim}$ with output $\yv = \pwa{\xv} \in \R^{\mdim}$ given an input $\xv \in \inset \subset \R^{\ndim}$.
This function is defined by the collection of $\npwa$ affine map tuples $\{(\hpoly(\Acon_1, \bcon_1), \Ccon_1, \dcon_1), \cdots, (\hpoly(\Acon_{\npwa}, \bcon_{\npwa}), \Ccon_{\npwa}, \dcon_{\npwa})\}$, such that
\begin{align}
    \pwa{\xv} &= \Ccon_\is \xv + \dcon_\is \quad \forall \xv \in \hpoly(\Acon_\is, \bcon_\is), \is = 1, \cdots, \npwa,
\end{align}
where $\Acon_\is \subset \R^{{\nhp}_{\is} \times \ndim}$, $\bcon_\is \subset \R^{{\nhp}_{\is}}$, $\Ccon_\is \subset \R^{\mdim \times \ndim}$, $\dcon_\is \subset \R^\mdim$.
We refer to each H-polytope $\hpoly(\Acon_\is, \bcon_\is)$ as a PWA region.

For the PWA system to be well-defined and continuous, we require
$\inset = \union_{\is = 1}^{\npwa} \hpoly(\Acon_{\is}, \bcon_{\is})$,
such that the input domain $\inset$ is \textit{tessellated} by the PWA regions, and that $\Ccon_\is \xv + \dcon_\is = \Ccon_\js \xv + \dcon_\js \forall \xv \in \hpoly(\Acon_\is, \bcon_\is) \cap \hpoly(\Acon_\js, \bcon_\js), \is, \js \in \{1, \cdots, \npwa\}$.
We refer to each pair of $\hpoly(\Acon_\is, \bcon_\is)$, $\hpoly(\Acon_\js, \bcon_\js)$ where $\exists \xv \in \hpoly(\Acon_\is, \bcon_\is) \cap \hpoly(\Acon_\js, \bcon_\js), \is, \js \in \{1, \cdots, \npwa\}$ as neighboring PWA regions.

Finally, per \cite{chung2024goal, vincent2021reachable}, the preimage $\brs{\cdot}$ of an H-polytope $\hpoly(\Acon, \bcon)$ through an affine map tuple $(\hpoly(\Acon_\is, \bcon_\is), \Ccon_\is, \dcon_\is)$ is:
\begin{subequations}
\begin{align}
    &\brs{\hpoly(\Acon, \bcon), (\hpoly(\Acon_\is, \bcon_\is), \Ccon_\is, \dcon_\is)} = \hpoly\left(\begin{bmatrix}
        \Acon\Ccon_{\is}\\
        \Acon_{\is}
    \end{bmatrix}, \begin{bmatrix}
        \bcon - \Acon\dcon_{\is}\\
        \bcon_{\is}
    \end{bmatrix}\right).
\end{align}
\end{subequations}

\subsection{ReLU Neural Networks and RPM}

In this work, we consider a fully connected, ReLU-activated feedforward neural network $\nn:\inset\to\R^{\mdim}$, with output $\yv = \nn(\xv) \in \R^\mdim$ given an input $\xv = \xv\idxo \in \inset \subset \R^{\ndim\idxo} = \R^\ndim$.
We denote by $\depth \in \N$ the \textit{depth} of the network and by $\ndim\idx{\is}$ the \textit{width} of the $\is^{\regtext{th}}$ layer.
Mathematically,
\begin{subequations}
\begin{align}
    \xv\idx{\is} &= \max\left(\Wt\idx{\is}\xv\idx{\is-1} + \bias\idx{\is}, \zeros\right),\\
    \yv &= \Wt\idx{\depth}\xv\idx{\depth-1} + \bias\idx{\depth},
\end{align}
\end{subequations}
where $\Wt\idx{\is} \in \R^{\ndim\idx{\is}\times \ndim\idx{\is-1}}$, $\bias\idx{\is} \in \R^{\ndim\idx{\is}}$, $\is=1,\cdots,\depth-1$, $\Wt\idx{\ds} \in \R^{\mdim\times \ndim\idx{\ds-1}}$, $\bias\idx{\ds} \in \R^{\mdim}$, and max is taken elementwise.

A ReLU neural network with this structure is \textit{equivalent} to a PWA system \cite{montufar2014number, hanin2019universal, arora2016understanding}.
Given a ReLU neural network $\nn$, we use the RPM algorithm \cite{vincent2021reachable} to obtain the affine map tuples of the equivalent PWA system.
For an input seed $\xv$, the first iteration of RPM returns:
\begin{align}\label{eq:first_map}
    \rpm{\nn, \xv, 1} &= (\hpoly(\Acon_{\sap_1}, \bcon_{\sap_1}), \Ccon_{\sap_1}, \dcon_{\sap_1}),
\end{align}
where $\xv \in \hpoly(\Acon_{\sap_1}, \bcon_{\sap_1})$, $\sap_1 \in \{1, \cdots, \npwa\}$.

For subsequent iterations $\is$, $\is = 2, \cdots, \npwa$, RPM returns:
\begin{align}
    \rpm{\nn, \xv, \is} &= (\hpoly(\Acon_{\sap_\is}, \bcon_{\sap_\is}), \Ccon_{\sap_\is}, \dcon_{\sap_\is}),
\end{align}
where $\sap_\is \in \{1, \cdots, \npwa\}$, $\sap_1 \neq \cdots \neq \sap_{\npwa}$, and $\hpoly(\Acon_{\sap_\is}, \bcon_{\sap_\is})$ is a neighboring PWA region of $\hpoly(\Acon_{\sap_\js}, \bcon_{\sap_\js})$, $\js\in\{1, \cdots, \is-1\}$.

Thus, RPM will eventually discover all the PWA regions and affine maps within the input domain, but can be terminated early if a region with some desired properties has been found.
RPM requires only basic matrix operations and solving LPs.
See \cite{vincent2021reachable} for detailed discussions of RPM.
\section{Problem Formulation} \label{subsec:blackbox_reachavoid}
In this paper, we consider the reach-avoid problem on trajectories realized by a black-box system using parameterized policies.
The trajectories are in the form of:
\begin{align}\label{eq:blackboxsys}
    \planv(\ts) &= \dyn_\pdim(\planv_0, \paramv, \ts, \dv),\ 
    \goiv_{\tf} = \dyn_\qdim(\paramv, \dv)\regtext{, and}\ 
    \planv(0) = \planv_0,
\end{align}
where $\planv \in \R^{\ndim_\ps}$ is the \textit{workspace state} (usually, $\ndim_\ps = 2$ or $3$), $\planv_0 \in \Xplan_0 \subset \R^{\ndim_\ps}$ is the initial workspace state, $\paramv \in \K \subset \R^{\nparam}$ represents \textit{trajectory parameters} (e.g. initial heading angle, desired goal position), $\ts \in [0, \tf]$ is time, $\tf\in\R^{+}$ is the final time, $\dv(\cdot)\in\Dfun:=\{\phi:[0, \tf]\to\D\}$ is the disturbance function, and $\goiv_{\tf} \in \Xgoi \subset \R^{\ndim_\qdim}$ are goal states of interest other than workspace (e.g. final heading angle, trajectory cost).

We assume $\dyn_\pdim$ and $\dyn_\qdim$ are continuous and \textit{black-box}, meaning we do not know their analytic expressions, but can observe the function's outputs by providing $\planv_0, \paramv, \ts$, and $\dv$ offline.
We also assume $\D$ and $\Xgoi$ are compact, $\Xplan_0$ and $\K$ are H-polytopes, and $0 \in \Xplan_0$.
Finally, we require $\dyn_\pdim$ to be \textit{translation invariant} in the workspace:
\begin{assum}[Translation Invariance in Workspace]
\label{ass:eti}
    Under the same $\paramv$, $\ts$, and $\dv$, changing $\planv_0$ translates the resulting trajectory by the same amount in the workspace:
    \begin{align}
        \dyn_\pdim(\planv_0, \paramv, \ts, \dv) = \dyn_\pdim(0, \paramv, \ts, \dv) + \planv_0.
    \end{align}
\end{assum}
\noindent This assumption is needed to enable computation of reachable sets for obstacle avoidance in workspace.

We denote the obstacles as $\Obs = \union_{\is = 1}^{\nobs} \Obs_\is \subset \R^{\ndim_\ps}$, and the goal set as $\Xgoal \subset \R^{\ndim_\ps} \times \Xgoi$.
$\Xgoal$ and each $\Obs_\is$ are represented as H-polytopes.
We define the reach-avoid problem as follows.
\begin{problem} [Backward Reach-Avoid Set (BRAS) for Black-Box Systems]\label{prob:bras}
    Given a black-box system \eqref{eq:blackboxsys}, $\Xgoal$, $\Obs$, $\tf$, $\Xplan_0$, and $\K$, find the BRAS $\Xi$, a set of initial workspace states in $\Xplan_0$ and trajectory parameters in $\K$ with which the robot reaches $\Xgoal$ at time $\tf$ without colliding with $\Obs$:
    \begin{equation}\label{eq:true_bras}
    \begin{split}
        \Xi(\tf, \Xgoal, \Obs) \subset
        \{&(\planv_0, \paramv) \in \Xplan_0 \times \K\ | \begin{bmatrix}\dyn_\pdim(\planv_0, \paramv, \tf, \dv)\\\dyn_\qdim(\paramv, \dv)\end{bmatrix} \in \Xgoal, \\
        &\dyn_\pdim(\planv_0, \paramv, \ts, \dv) \notin \Obs,\dyn_\pdim(\planv_0, \paramv, 0, \dv) = \planv_0,\\
        &\forall \ts\in[0, \tf], \dv(\cdot)\in\Dfun\}.
    \end{split}
    \end{equation}
\end{problem}
\noindent Offline, we assume we know $\tf$, $\Xplan_0$, and $\K$, and are allowed to interact with \eqref{eq:blackboxsys}.
We only obtain $\Xgoal$ and $\Obs$ online.
\section{Proposed Method} \label{sec:method}

We now detail our approach to solving Problem \ref{prob:bras}.
Offline, we build a trajectory model of the black-box system using ReLU neural networks.
We then compute an upper bound of the modeling error by repeatedly interacting with the black-box system.
Online, we use RPM \cite{vincent2021reachable} to convert the neural network model into a PWA system, and use NeuralPARC to incorporate the modeling error and compute the BRAS.

\subsection{Learning the Trajectory Model (Offline)}
The key idea of NeuralPARC is to estimate and represent the black-box system's trajectories in a more analyzable form without losing representation power and without requiring knowledge about the system dynamics.
Our key insight is to accomplish this using ReLU neural networks.
Offline, we uniformly sample $\paramv$ from $\K$, and for each sample, collect $\planv(\ts)$ and $\goiv_{\tf}$ from \eqref{eq:blackboxsys}, with $\planv_0 = 0$, a random $\dv(\cdot)\in\Dfun$, and $\ts = \Delta\ts, \cdots, \tf$, where $\Delta\ts$ is the timestep, $0 < \Delta\ts < \tf$ and $\tf~\regtext{mod}~\Delta\ts = 0$.
Then, we train a ReLU neural network $\nn$ with features $(\paramv)$ and labels $(\planv(\Delta\ts), \cdots, \planv(\tf), \goiv_{\tf})$ in a supervised learning framework.
From translation invariance, we have
\begin{align} \label{eq:traj_model}
    \nn(\paramv) + [\planv_0\tp, \cdots, \planv_0\tp, \zeros]\tp = [\hat{\planv}(\Delta\ts)\tp, \cdots, \hat{\planv}(\tf)\tp, \hat{\goiv}_{\tf}\tp]\tp,
\end{align}
where $\hat{\planv}(\ts) \in \R^{\ndim_\ps}$ is the estimation of $\dyn_\pdim(\planv_0, \paramv, \ts, \dv)$, and $\hat{\goiv}_{\tf} \in \R^{\ndim_\qdim}$ is the estimation of $\dyn_\qdim(\paramv, \dv)$.
By definition, $\hat{\planv}(0) = \planv_0$.
We denote $\hat{\planv}(\ts)$ and $\hat{\goiv}_{\tf}$ as the \textit{trajectory model}.

In the equivalent PWA form of $\nn$, we have
\begin{align}\label{eq:pwa_map}
    \hat{\planv}(\ts) = \Ccon_{\is, \ts}[\planv_0\tp, \paramv\tp]\tp + \dcon_{\is, \ts}
    \ \regtext{and}\
    \hat{\goiv}_{\tf} = \Ccon_{\is, \qdim}[\planv_0\tp, \paramv\tp]\tp + \dcon_{\is, \qdim}
\end{align}
for all $[\planv_0\tp, \paramv\tp]\tp\in\hpoly([\zeros, \Acon_\is], \bcon_\is)$, where for $\ts = 0$, $\Ccon_{\is, 0} = [\eye_{\ndim_\ps}, \zeros]$, $\dcon_{\is, 0} = \zeros$, for $\ts = \Delta\ts, \cdots, \tf$, $\Ccon_{\is, \ts} = [\eye, (\Ccon_\is)_{\ell_1:\ell_2, 1:\nparam}]$, $\dcon_{\is, \ts} = (\dcon_\is)_{\ell_1:\ell_2}$, $\ell_1 = (\frac{\ts}{\Delta\ts} - 1)\ndim_\ps + 1$, $\ell_2 = \frac{\ts}{\Delta\ts}\ndim_\ps$, $\Ccon_{\is, \qdim} = [\zeros , (\Ccon_\is)_{\ell_3:\ell_4, 1:\nparam}]$, $\dcon_{\is, \qdim} = (\dcon_\is)_{\ell_3:\ell_4}$, $\ell_3 = \frac{\tf}{\Delta\ts}\ndim_\ps + 1$, $\ell_4 = \frac{\tf}{\Delta\ts}\ndim_\ps + \ndim_{\qdim}$, and $(\hpoly(\Acon_\is, \bcon_\is), \Ccon_\is, \dcon_\is)$ is the $\is^\regtext{th}$ affine map tuple of the equivalent PWA system, $\is = 1, \cdots, \npwa$.

To avoid obstacles between timesteps, we use linear interpolation to approximate continuous time motion.
At $\ts = 0, \Delta\ts, \cdots, \tf$, $\hat{\planv}(\ts')$ for $\ts < \ts' < \ts + \Delta\ts$ is defined by
\begin{align}\label{eq:cont_time_approx}
    \hat{\planv}(\ts') = \hat{\planv}(\ts) + (\hat{\planv}(\ts + \Delta\ts) - \hat{\planv}(\ts))(\tfrac{\ts' - \ts}{\Delta\ts}).
\end{align}

\subsection{Estimating Modeling Error (Offline)}
To account for discrepancies between the trajectory model and the actual trajectories realized by the black-box system, we require a \textit{modeling error bound}.
We estimate the modeling error bound similarly to PARC \cite{chung2024goal} and \cite{kousik2019safe, shao2021reachability}.

Offline, we sample $\planv_{0, \is}$, $\paramv_\is$, and $\dv_\is(\cdot)$ uniformly from $\Xplan_0$, $\K$, and $\Dfun$ for $\is = 1, \cdots, \ndim_\sample$.
Then, for each $\js = 1, \cdots, \ndim_\ps + \ndim_{\qdim}$, we define the \textit{maximum final error} $\es_{\tf, \js} \in \R^{+}_{0}$ as:
\begin{align}\label{eq:max_final_error}
\es_{\tf, \js} = \max_\is
        \abs{\left([
        \hat{\planv}(\tf)\tp,
        \hat{\goiv}_{\tf}\tp
        ]\tp
        - [
        \planv(\tf)\tp,
        \goiv_{\tf}\tp
        ]\tp
        \right)_{\js}},
\end{align}
and for each $\js = 1, \cdots, \ndim_\ps$ and each $\ts = 0, \Delta\ts, \cdots, \tf-\Delta\ts$, we define the \textit{maximum interval error} $\itv{\es}_{\ts, \js} \in \R^{+}_{0}$ as:
\begin{align}\label{eq:max_int_error}
    \itv{\es}_{\ts, \js} = \max_{\is,\ \ts' \in [\ts,\ts+\Delta\ts]}\abs{\left(\hat{\planv}(\ts') - \planv(\ts') \right)_{\js}},
\end{align}
where $\hat{\planv}$, $\hat{\goiv}_{\tf}$, $\planv$, and $\goiv_{\tf}$ are computed from \eqref{eq:traj_model} and \eqref{eq:blackboxsys} with $\paramv_\is$, $\planv_{0, \is}$, $\dv_\is$.
We assume this approach provides an upper bound to the modeling error.

To incorporate the error bounds for geometric computation, we express the \textit{maximum final error set} $\Eset_{\tf} \subset \R^{\ndim_\ps + \ndim_{\qdim}}$ and the \textit{maximum interval error set} $\itv{\Eset}_{\ts} \subset \R^{\ndim_\ps}$ as hyperrectangles centered at the origin for $\ts = 0, \Delta\ts, \cdots, \tf-\Delta\ts$:
\begin{subequations}
\begin{align}
    \Eset_{\tf} &=  [-\es_{\tf, 1}, \es_{\tf, 1}]\times\cdots\times[-\es_{\tf, \ndim_\ps + \ndim_{\qdim}}, \es_{\tf, \ndim_\ps + \ndim_{\qdim}}],\\
    \itv{\Eset}_{\ts} &= [-\itv{\es}_{\ts, 1}, \itv{\es}_{\ts, 1}]\times\cdots\times[-\itv{\es}_{\ts, \ndim_\ps}, \itv{\es}_{\ts, \ndim_\ps}].
\end{align}
\end{subequations}
We refer the reader to {\cite[Section IV.D]{chung2024goal}} for a detailed discussion on the validity of our approach and assumptions on the modeling error.
In the following sections, we will discuss how to compute \eqref{eq:subset_bras} using NeuralPARC.

\subsection{NeuralPARC: Reach Set (Online)}\label{sec:reach_set}
In this section, we compute the backward reachable set (BRS) of the trajectory model with respect to an affine map.
The BRS is a set of initial workspace conditions $\planv_0$ and trajectory parameters $\paramv$ with which the black-box system is \textit{guaranteed} to reach the goal set $\Xgoal$ at time $\tf$.

First, shrink or buffer the goal and obstacles as $\tilde{\Xgoal} = \Xgoal \ominus \Eset_{\tf}$ and $\tilde{\Obs}_{\ts} = \union_{\is = 1}^{\nobs} \tilde{\Obs}_{\ts,\is} = \union_{\is = 1}^{\nobs} ({\Obs}_{\is} \oplus \itv{\Eset}_{\ts})$ for $\ts = 0, \cdots, \tf-\Delta\ts$.
\begin{lem}[Translating Guarantees from Trajectory Model to Black-Box System]\label{lem:traj_to_real}
    If \eqref{eq:max_final_error} and \eqref{eq:max_int_error} provides an upper bound to the modeling error, then the BRAS $\hat{\Xi}$ of the trajectory model:
    \begin{equation} \label{eq:subset_bras}
    \begin{split}
        \hat{\Xi}(\tf, \Xgoal, \Obs) =
        \{&(\planv_0, \paramv) \in \Xplan_0 \times \K\ |\begin{bmatrix}
        \hat{\planv}(\tf)\\
        \hat{\goiv}_{\tf}
        \end{bmatrix} \in \tilde{\Xgoal}, \eqref{eq:traj_model}, \hat{\planv}(0) = \planv_0\\
        &\hat{\planv}(\ts') \notin \tilde{\Obs}_{\ts}, \ts \leq \ts' \leq \ts + \Delta\ts,\forall \ts = 0, \cdots, \tf-\Delta\ts\},
    \end{split}
    \end{equation}
    fulfills \eqref{eq:true_bras}, which is defined on the black-box system \eqref{eq:blackboxsys}.
\end{lem}
\begin{proof}
    See \cite[Section IV.E]{chung2024goal}.
\end{proof}
\noindent We can now perform analysis directly on the trajectory model. 
To begin, we use RPM to obtain $(\hpoly(\Acon_{\sap_1}, \bcon_{\sap_1}), \Ccon_{\sap_1}, \dcon_{\sap_1})$ from \eqref{eq:first_map} for $\nn$ and a random seed $\paramv \in \K$.
Then, the BRS $\Rbrs$ is an H-polytope \cite{chung2024goal, thomas2006robust}:
\begin{prop} [BRS of an Affine Map] \label{prop:brs}
    Given $(\hpoly(\Acon_{\sap_\is}, \bcon_{\sap_\is}), \Ccon_{\sap_\is}, \dcon_{\sap_\is})$ and $\Xgoal$.
    The reach set $\Rbrs$ of $\Xgoal$ in the $\is^\regtext{th}$ PWA region is:
    \begin{align}\label{eq:brs}
        \Rbrs = \brs{\tilde{\Xgoal}, (\Xplan_0 \times \hpoly(\Acon_{\sap_\is}, \bcon_{\sap_\is}), \begin{bmatrix}\Ccon_{\sap_\is, \tf}\\\Ccon_{\sap_\is, \qdim}\end{bmatrix}, \begin{bmatrix}\dcon_{\sap_\is, \tf}\\\dcon_{\sap_\is, \qdim}\end{bmatrix})}.
    \end{align}
    Then, for all $[{\planv_0}\tp, {\paramv}\tp]\tp \in \Rbrs$, $\dv(\cdot)\in\Dfun$, we have $[\dyn_\pdim(\planv_0, \paramv, \tf, \dv)\tp, \dyn_\qdim(\paramv, \dv)\tp]\tp \in \Xgoal$.
\end{prop}
\begin{proof}
    See \cite[Proposition 17]{chung2024goal}.
\end{proof}
\noindent Once the BRS is computed, we check whether it is empty by solving an LP.
If it is not, we proceed to the next section to account for obstacle avoidance.
If the BRS is empty, we query RPM to return a neighboring PWA region and its affine map tuple to repeat the analysis in this section.

\subsection{NeuralPARC: Avoid Set (Online)}\label{sec:avoid_set}
In this section, we compute the backward avoid set (BAS) of the trajectory model with respect to an affine map.
The BAS \textit{overapproximates} the set of all initial workspace conditions $\planv_0$ and trajectory parameters $\paramv$ with which the black-box system will collide with the obstacles $\Obs$ at some time $\ts \in [0, \tf]$.

With the affine map tuple $(\hpoly(\Acon_{\sap_\is}, \bcon_{\sap_\is}), \Ccon_{\sap_\is}, \dcon_{\sap_\is})$ from Section \ref{sec:reach_set}, we can now compute the BAS $\Ravd$ of obstacles $\Obs$ through an affine map as a union of AH-polytopes.
\begin{thm} [BAS of an Affine Map] \label{thm:bas}
    Given $(\hpoly(\Acon_{\sap_\is}, \bcon_{\sap_\is}), \Ccon_{\sap_\is}, \dcon_{\sap_\is})$, $\Rbrs$, and, $\Obs$, the BAS $\Ravd$ of $\Obs$ in the $\is^\regtext{th}$ PWA region is:
    \begin{align}
        \Ravd &= \union_{\js = 1}^{\nobs} \union_{\ts\in\{0, \cdots, \tf-\Delta\ts\}} \left(\Rbrs\cap\left(\conv{\Obsbrs_{\ts, \js}, \Obsbrs_{\ts+\Delta\ts, \js}}\times\R^{\nparam}\right)\right),
    \end{align}
    where $\Obsbrs_{\ts, \js} = \projtext{\nplan}(\brsfunction(\tilde{\Obs}_{\ts, \js}, (\hpoly([\zeros, \Acon_{\sap_\is}], \bcon_{\sap_\is}), \Ccon_{\sap_\is, \ts}, \dcon_{\sap_\is, \ts})))$ and $\Obsbrs_{\ts+\Delta\ts, \js} = \projtext{\nplan}(\brsfunction(\tilde{\Obs}_{\ts, \js}, (\hpoly([\zeros, \Acon_{\sap_\is}], \bcon_{\sap_\is}), \Ccon_{\sap_\is, \ts+\Delta\ts}, \dcon_{\sap_\is, \ts+\Delta\ts})))$.
    Then, for all $[{\planv_0}\tp, {\paramv}\tp]\tp \in \Rbrs\setminus\Ravd$, $\dv(\cdot)\in\Dfun$, $\ts\in[0, \tf]$ we have $\dyn_\pdim(\planv_0, \paramv, \ts, \dv) \notin \Obs$.
\end{thm}
\begin{proof}
    By Lemma \ref{lem:traj_to_real}, it is sufficient to show that, if there is a point $[{\planv_0}\tp, \paramv\tp]\tp$ in $\Rbrs$ that collides with $\tilde{\Obs}_{\ts, \js}$ between time $\ts$ and $\ts + \Delta\ts$, then $\planv_0$ must be in the convex hull between $\Obsbrs_{\ts, \js}$ and $\Obsbrs_{\ts+\Delta\ts, \js}$.
    Mathematically, we want to show if $\exists [{\planv_0}\tp, \paramv\tp]\tp \in \Rbrs, \alpha \in [0, 1]$ such that $\hat{\planv}(\ts) + \alpha(\hat{\planv}(\ts + \Delta\ts) - \hat{\planv}(\ts)) \in \tilde{\Obs}_{\ts, \js}$, then $\exists \beta\in[0, 1], \planv_1\in\Obsbrs_{\ts, \js}, \planv_2\in\Obsbrs_{\ts+\Delta\ts, \js}$ such that $\planv_0 = \planv_1 + \beta(\planv_2 - \planv_1)$. 
    From Assumption \ref{ass:eti} and \eqref{eq:pwa_map}, we can always find some $\planv_1 + \Ccon_\ts \paramv + \dcon_\ts \in \tilde{\Obs}_{\ts, \js}$ and $\planv_2 + \Ccon_{\ts + \Delta\ts} \paramv + \dcon_{\ts + \Delta\ts} \in \tilde{\Obs}_{\ts, \js}$, where $\Ccon_\ts = (\Ccon_{\sap_\is, \ts})_{1:\nplan, (\nplan+1):\nparam}$, $\dcon_\ts = (\dcon_{\sap_\is, \ts})_{(\nplan+1):\nparam}$, $\Ccon_{\ts + \Delta\ts} = (\Ccon_{\sap_\is, \ts + \Delta\ts})_{1:\nplan, (\nplan+1):\nparam}$, and $\dcon_{\ts + \Delta\ts} = (\dcon_{\sap_\is, \ts + \Delta\ts})_{(\nplan+1):\nparam}$, which implies $\planv_1\in\Obsbrs_{\ts, \js}$ and $\planv_2\in\Obsbrs_{\ts+\Delta\ts, \js}$.
    Then, the proof is satisfied when $\beta = \alpha$, $\planv_1 = \planv_0 + \alpha\hat{\planv}_d$, and $\planv_2 = \planv_0 + (\alpha-1)\hat{\planv}_d$, where $\hat{\planv}_d = ((\Ccon_{\ts + \Delta\ts} - \Ccon_{\ts})\paramv + \dcon_{\ts + \Delta\ts} - \dcon_{\ts})$.
\end{proof}
\noindent From Proposition \ref{prop:brs} and Theorem \ref{thm:bas}, every $[{\planv_0}\tp, \paramv\tp]\tp \in \Rbrs\setminus\Ravd$ guarantees goal-reaching and obstacle avoidance (i.e. $\Rbrs\setminus\Ravd$ is a BRAS).
To sample from $\Rbrs\setminus\Ravd$, we first sample $[{\planv_0}\tp, \paramv\tp]\tp$ from within $\Rbrs$, where checking $[{\planv_0}\tp, \paramv\tp]\tp\in\Rbrs$ requires only an inequality check.
Then, we check whether $\planv_0$ is in $\Ravd_{\nplan} := \proj[\nplan]{\Ravd}$, which are $\nobs \times \frac{\tf}{\Delta\ts}$ AH-polytopes given by:
\begin{align}
    \Ravd_{\nplan} &= \union_{\js = 1}^{\nobs} \union_{\ts\in\{0, \cdots, \tf-\Delta\ts\}} \left(\proj[\nplan]{\Rbrs}\cap\left(\conv{\Obsbrs_{\ts, \js}, \Obsbrs_{\ts+\Delta\ts, \js}}\right)\right),
\end{align}
$\Ravd_{\nplan}$ only requires basic matrix operations to compute.
If $\planv_0$ is in any of the AH-polytopes (which can be checked by solving LPs), then $[{\planv_0}\tp, \paramv\tp]\tp \notin \Rbrs\setminus\Ravd$.
We can speed up this process by stopping once an AH-polytope that contains $\planv_0$ has been found, or by first rooting out empty AH-polytopes.

If a point from within $\Rbrs\setminus\Ravd$ cannot be found after a certain amount of samples, or discovering a larger variety of safe trajectories is desired, we can query RPM to return a different PWA region, and repeat the steps in Section \ref{sec:reach_set} and Section \ref{sec:avoid_set} until all PWA regions have been explored.
\begin{figure}[t]
\centering
    \includegraphics[width=1\columnwidth]{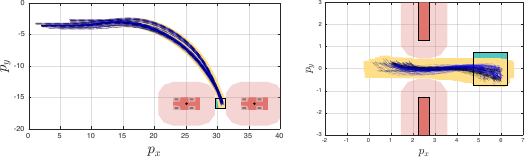}
\caption{
100 samples of (left) the drift parking vehicle and (right) the ASV from NeuralPARC using a network of depth 5 and width 8 for each hidden layer.
The yellow tubes are the modeling error bounds buffered onto NeuralPARC's predicted trajectories (dashed lines).
All actual trajectories (solid lines) reach the green goal and avoid the red obstacles buffered with the agent's circular volume.
}\label{fig:traj_combined}
\vspace*{-0.5cm}
\end{figure}

\section{Experiments} \label{sec:result}
We now compare the performance of NeuralPARC to PARC \cite{chung2024goal} on drift parallel parking maneuvers, and assess how network size affects NeuralPARC's performance.
All experiments, demonstrations, and training were run on a desktop computer with a 24-core i9 CPU, 32 GB RAM, and an Nvidia RTX 4090 GPU on MATLAB.

\subsection{Experiment Setup}
To ensure a fair comparison with PARC, we used the same parameterized drifting model as \cite{chung2024goal}, with $\planv = [\px, \py]\tp\subset\R^2 = \Xplan_0$, where $\px$ and $\py$ are the x and y-coordinate of the center of the car, $\paramv = [\vel, \pth_{\sideslip}]\tp \subset [9, 11]\times[\frac{\pi}{6}, \frac{2}{9}\pi]$, where $\vel$ is the desired velocity to enter the drifting regime, and $\pth_{\sideslip}$ is the desired angle to perform a hard braking, $\goiv_{\tf} = \pth_{\tf}$, where $\pth_{\tf}$ is the yaw angle of the car at time $\tf$, $\tf = 7.8$, $\Delta\ts = 0.1$, $\D = \emptyset$ to ensure a fair comparison (the original PARC example does not include disturbance), $\Xgoal = [29.7, 31.3]\times[-16.8, -15.2]\times[\frac{5}{6}\pi, \frac{7}{6}\pi]$, and $\Obs_1$, $\Obs_2$ are shown in Fig.~\ref{fig:traj_combined}.
To account for the car's volume, we Minkowski-summed $\Obs$ with the circular overapproximation of the car's body.
See \cite{chung2024goal} for details of the system dynamics and PARC's implementation.

Offline, we collected data from 10,000 trajectories by uniformly sampling $\K$, from which the methods will build the trajectory model on.
For NeuralPARC, the ReLU neural networks used have a depth of 5, with the width of each hidden layer varying between 6, 7, and 8 to observe its effects on the performance.

Online, since NeuralPARC's performance on locating the first safe sample depends on the randomized initial seed of $\paramv$, we ran NeuralPARC on the same environment setup 100 times with different seeds.
In each PWA region, NeuralPARC obtained 50 samples from the BRS.
If none of the samples were in the BRAS, NeuralPARC would explore the neighboring PWA region.

\begin{table}[t]
\captionsetup{font=small}
\centering
\begin{tabular}{l|r|r|r|r|r}
     \multicolumn{1}{c|}{\textbf{Method}} & \multirow{2}{*}{\textbf{$\npwa$}} & \multicolumn{1}{c|}{\textbf{BRAS}} & \multicolumn{3}{c}{\textbf{Time Until First Sample (\unit{s})}} \\
    \textbf{($\ndim\idx{1}, \cdots, \ndim\idx{\depth-1}$)} & & \textbf{Time (\unit{s})} & \multicolumn{1}{c|}{\textbf{Min}} & \multicolumn{1}{c|}{\textbf{Max}} & \multicolumn{1}{c}{\textbf{Mean}}\\
    \hline
    \multicolumn{6}{c}{\textbf{Extreme Drift Parallel Parking}}\\
    \hline
    \multirow{2}{*}{PARC} & \multirow{2}{*}{1} & \multirow{2}{*}{0.82} & \multirow{2}{*}{Timeout} & \multirow{2}{*}{Timeout} & \multirow{2}{*}{Timeout}\\
    &&&&&\\
    NeuralPARC, & \multirow{2}{*}{68} & \multirow{2}{*}{3.12} & \multirow{2}{*}{0.73} & \multirow{2}{*}{4.61} & \multirow{2}{*}{3.14}\\
    (6, 6, 6, 6) &&&&&\\
    NeuralPARC, & \multirow{2}{*}{122} & \multirow{2}{*}{2.47} & \multirow{2}{*}{0.46} & \multirow{2}{*}{3.61} & \multirow{2}{*}{2.45}\\
    (7, 7, 7, 7) &&&&&\\
    NeuralPARC, & \multirow{2}{*}{203} & \multirow{2}{*}{9.59} & \multirow{2}{*}{0.29} & \multirow{2}{*}{4.85} & \multirow{2}{*}{1.90}\\
    (8, 8, 8, 8) &&&&&\\
    \hline
    \multicolumn{6}{c}{\textbf{Deep RL ASV Agent with Large Disturbances}}\\
    \hline
    NeuralPARC, & \multirow{2}{*}{611} & \multirow{2}{*}{100.02} & \multirow{2}{*}{1.42} & \multirow{2}{*}{165.53} & \multirow{2}{*}{71.1}\\
    (8, 8, 8, 8) &&&&&
\end{tabular}
\caption{NeuralPARC and PARC \cite{chung2024goal} compared across different systems and network sizes.
The BRAS time is the time the method takes to compute the BRAS for \textit{all} PWA regions.
The time until first sample is the time the method takes to identify the first safe initial workspace state and trajectory parameter for 100 attempts.}
\label{table:experiments}
\vspace*{-0.5cm}
\end{table}

\subsection{Results and Discussion}
The results of the experiments are shown in Table \ref{table:experiments}, Fig.~\ref{fig:traj_combined}, Fig.~\ref{fig:car_pwa}, and Fig.~\ref{fig:car_error}.
In all experiments, trajectories identified by NeuralPARC all reached the goal and avoided the obstacles, whereas PARC failed to find any safe plans due to the large modeling error bound.

The number of PWA regions and therefore BRAS computation time across all PWA regions increases with the network size.
On the other hand, the larger the network size, the smaller the modeling error, and therefore the larger the BRAS volume and the more varied the safe trajectories are.

Surprisingly, despite the increase in the number of PWA regions, the time it took for one safe sample to be found \textit{decreases} with increase in network size.
Likely, the decrease in modeling error allows more PWA regions with safe samples to be discovered, offsetting the time required to explore more PWA regions.

\section{Demonstration}
The key idea of NeuralPARC is to distill black-box trajectory behaviors into the powerful, yet analyzable form of ReLU neural network to enable reach-avoid guarantees.
Thus, NeuralPARC is agnostic to how trajectories are generated, and can attain good performance as long as they are well-behaved and modellable.
We now demonstrate this versatility applying NeuralPARC to drift parking on robot car hardware, and navigating a narrow gap on a simulated ASV trained with deep RL.

\begin{figure}
\centering
    \includegraphics[width=1\columnwidth]{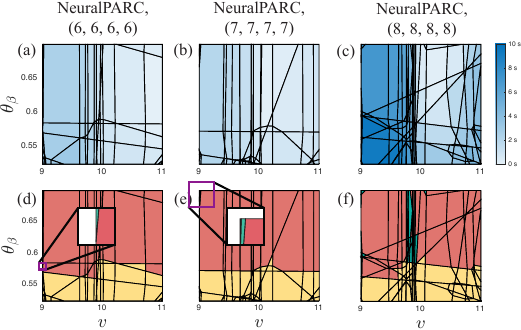}
\caption{
PWA regions in $\K$ for different network sizes in NeuralPARC, visualized with (a)-(c) BRAS computation time, where color indicates time elapsed, and (d)-(f) success of finding a safe sample, where red indicates failure, green indicates success, and yellow indicates an empty BRS.
The green regions for (d) and (e), highlighted by purple squares, are zoomed in for clarity.
}\label{fig:car_pwa}
\vspace*{-0.5cm}
\end{figure}

\subsection{Hardware Demonstration Setup}
For the hardware demonstration, we designed a parameterized family of drift parallel parking maneuvers on an F1/10 class robotic vehicle \cite{o2019f1}, with $\planv = [\px, \py]\tp\subset\R^2 = \Xplan_0$, where $\px$ and $\py$ are the x and y-coordinate of the center of the car in \unit{mm}, $\paramv = [\pydes, \vdes]\tp \subset [-1.9, -1]\times[2, 3.25]$, where $\pydes$ and $\vdes$ are the target y-coordinate and velocity in the motion capture system to begin drifting, $\goiv_{\tf} = \pth_{\tf}$, where $\pth_{\tf}$ is the yaw angle of the car in \unit{rad} at time $\tf$, $\tf = 4.9\unit{s}$, $\Delta\ts = 0.1\unit{s}$, $\Xgoal = [290, 850]\times[-4,225, -3,615]\times[-\frac{4}{3}\pi, -\frac{2}{3}\pi]$, and $\Obs_1$, $\Obs_2$ are two F1/10 vehicles of the same size, placed $1,200 \unit{mm}$ apart, centered at $\Xgoal$ (see Fig.~\ref{fig:front_figure}).
We accounted for the vehicle's volume in the same manner as in Section \ref{sec:result}.
See \cite{o2019f1} for the specifications of the vehicle used.

Offline, we uniformly sample 60 parameters from $\K$.
Trajectory data were collected by first using an OptiTrack $120 \unit{Hz}$ motion capture system for closed-loop proportional-integral-derivative (PID) feedback to arrive at the sampled $\pydes, \vdes$. 
Then, an open-loop drifting maneuver was performed.
The trajectory model was trained on a ReLU neural network with depth of 3 and hidden layer width of 8.
Online, we used NeuralPARC to identify one safe trajectory, which we tasked the vehicle to repeatedly follow for 10 times.

\subsection{ASV Demonstration Setup}
For the ASV demonstration, we used the model and deep RL agent trained in \cite{batista2024deep}, with $\planv = [\px, \py]\tp\subset\R^2 = \Xplan_0$, where $\px$ and $\py$ are the x and y-coordinate of the center of the boat, $\paramv = [\pth_0, \pxg, \pyg]\tp \subset [-\frac{\pi}{6}, \frac{\pi}{6}]\times[4, 7]\times[-1, 1]$, where $\pth_0$ is the heading angle of the boat, $\pxg$ and $\pyg$ are the x and y-coordinate of the desired goal position, $\goiv_{\tf} = \emptyset$, $\tf = 10$, $\Delta\ts = 0.1$, and the location of $\Xgoal$, $\Obs_1$, and $\Obs_2$ shown in Fig.~\ref{fig:traj_combined}.
Disturbances were applied on mass, force, torque, position, velocity, heading angle, and control inputs of the ASV.
We accounted for the boat's volume in the same manner as in Section \ref{sec:result}.
The RL policy was trained to reach the goal \textit{without} avoiding obstacles.
See \cite{batista2024deep} for details of the system dynamics, disturbances, and RL training.

Offline, we collected data from 5,000 trajectories of the RL agent in Isaac Gym \cite{makoviychuk2021isaac} to train the trajectory model.
We used a ReLU neural network with depth of 5 and hidden layer width of 8.
Online, we ran NeuralPARC 100 times in the same manner as Section \ref{sec:result}.
For practicality in performance, we partitioned $\K$ into 3-by-3-by-3 hyperrectangles, and computed the modeling error separately for each subdomain.
We refer the readers to \cite[Section III.C]{kousik2019technical} for the theoretic justifications and implementation strategy of partitioning the error bound for different parameter subdomains.

\subsection{Results and Discussion}
The results of the hardware experiment are shown in Fig.~\ref{fig:front_figure} with the first safe trajectory found after $13.44 \unit{s}$, while the results of the ASV experiment are shown in Table \ref{table:experiments}, Fig.~\ref{fig:front_figure}, and Fig.~\ref{fig:traj_combined}.
In all experiments, robots guided by NeuralPARC all reached the goal and avoided the obstacles.

Due to the open-loop drifting maneuver, the realized motion of the vehicle varied significantly even with the same trajectory parameter. 
Moreover, the limited amount of data collectible from hardware demonstrations made it very difficult to construct an accurate trajectory and error model.
On the other hand, while data can be collected easily from simulations, the large disturbances to the ASV and the unpredictable nature of RL made establishing formal guarantees extremely challenging \cite{qin2022sablas}.
Despite this, by accounting for the worst-case error and effective representation of reachable sets, NeuralPARC can find conditions where the sweep of the error bound with the predicted trajectories just barely skirt by the obstacles, as shown in Fig.~\ref{fig:front_figure}, successfully maintaining both \textit{safety} and \textit{liveness}.

\begin{figure}[t]
\centering
    \includegraphics[width=1\columnwidth]{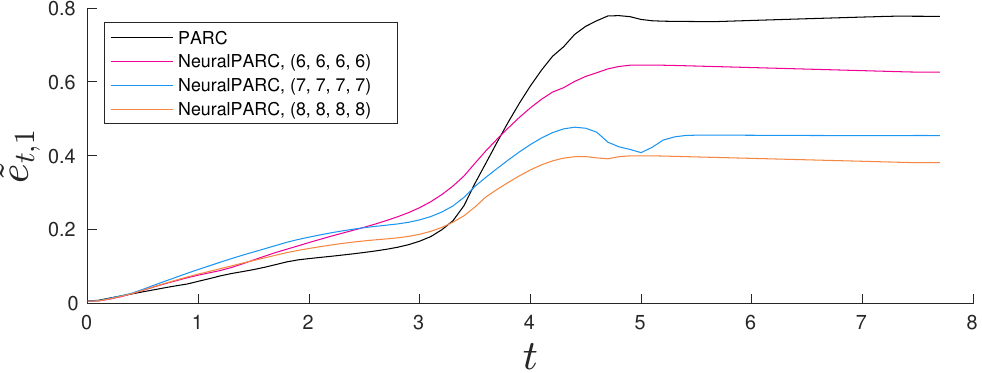}
\caption{
\textit{Maximum interval error} of $\px$ for PARC and NeuralPARC with different network sizes.
}\label{fig:car_error}
\vspace*{-0.5cm}
\end{figure}

\section{Conclusion} \label{sec:conclusion}
This paper proposes NeuralPARC, a method for solving the reach-avoid problem in black-box systems.
Our approach involves distilling the behavior of black-box trajectories into the expressive, yet analyzable form of ReLU neural networks, while correctly accounting for the modeling error.
We validated our approach on extreme drift parallel parking maneuvers and on a deep RL ASV agent.

NeuralPARC has two key limitations.
Firstly, its performance depends heavily on the starting seed.
Since RPM explores PWA regions neighbor-by-neighbor, the time until a safe sample is found can be very long if the trajectory corresponding to the starting seed is far away from safety.
As such, a future direction could be to use fast, but not necessarily always safe reach-avoid methods \cite{chen2021fastrack, kousik2020bridging} to determine a starting seed for NeuralPARC as a \textit{warm start}, enabling NeuralPARC for real-time applications.

Secondly, like PARC \cite{chung2024goal}, the validity of NeuralPARC depends heavily on the number of samples used in the modeling error computation, which limits its applicability in systems where data are not readily available.
If the modeling error was indeed found to be larger than expected online, one could update the modeling error bound, or fine-tune the learned trajectory model with the new data \cite{mamakoukas2020memory, miller2015ergodic}.

\renewcommand{\bibfont}{\normalfont\footnotesize}
{\renewcommand{\markboth}[2]{}
\printbibliography}

@article{nguyen2011model,
  title={\href{https://doi.org/10.1007/s10339-011-0404-1}{Model learning for robot control: a survey}},
  author={Nguyen-Tuong, Duy and Peters, Jan},
  journal={Cognitive processing},
  volume={12},
  pages={319--340},
  year={2011},
  publisher={Springer}
}

@inproceedings{vincent2021reachable,
  title={\href{https://doi.org/10.1109/ICRA48506.2021.9561956}{Reachable polyhedral marching (rpm): A safety verification algorithm for robotic systems with deep neural network components}},
  author={Vincent, Joseph A and Schwager, Mac},
  booktitle={2021 IEEE International Conference on Robotics and Automation (ICRA)},
  pages={9029--9035},
  year={2021},
  organization={IEEE}
}

@INPROCEEDINGS{chung2024goal, 
    AUTHOR    = {Long Kiu Chung AND Wonsuhk Jung AND Chuizheng Kong AND Shreyas Kousik}, 
    TITLE     = {\href{https://doi.org/10.15607/RSS.2024.XX.117}{Goal-Reaching Trajectory Design Near Danger with Piecewise Affine Reach-avoid Computation}}, 
    BOOKTITLE = {Proceedings of Robotics: Science and Systems}, 
    YEAR      = {2024}, 
    ADDRESS   = {Delft, Netherlands}, 
    MONTH     = {July}, 
    DOI       = {10.15607/RSS.2024.XX.117} 
}

@inproceedings{herceg2013multi,
  title={\href{https://doi.org/10.23919/ECC.2013.6669862}{Multi-parametric toolbox 3.0}},
  author={Herceg, Martin and Kvasnica, Michal and Jones, Colin N and Morari, Manfred},
  booktitle={2013 European control conference (ECC)},
  pages={502--510},
  year={2013},
  organization={IEEE}
}

@inproceedings{sadraddini2019linear,
  title={\href{https://doi.org/10.1109/CDC40024.2019.9029363}{Linear encodings for polytope containment problems}},
  author={Sadraddini, Sadra and Tedrake, Russ},
  booktitle={2019 IEEE 58th conference on decision and control (CDC)},
  pages={4367--4372},
  year={2019},
  organization={IEEE}
}

@article{thomas2006robust,
  title={\href{https://doi.org/10.3182/20060607-3-IT-3902.00061}{Robust model predictive control for piecewise affine systems subject to bounded disturbances}},
  author={Thomas, Jean and Olaru, Sorin and Buisson, Jean and Dumur, Didier},
  journal={IFAC Proceedings Volumes},
  volume={39},
  number={5},
  pages={329--334},
  year={2006},
  publisher={Elsevier}
}

@inproceedings{tran2019star,
  title={\href{https://doi.org/10.1007/978-3-030-30942-8_39}{Star-based reachability analysis of deep neural networks}},
  author={Tran, Hoang-Dung and Manzanas Lopez, Diago and Musau, Patrick and Yang, Xiaodong and Nguyen, Luan Viet and Xiang, Weiming and Johnson, Taylor T},
  booktitle={Formal Methods--The Next 30 Years: Third World Congress, FM 2019, Porto, Portugal, October 7--11, 2019, Proceedings 3},
  pages={670--686},
  year={2019},
  organization={Springer}
}

@article{selim2022safe,
  title={\href{https://doi.org/10.1109/LRA.2022.3192205}{Safe reinforcement learning using black-box reachability analysis}},
  author={Selim, Mahmoud and Alanwar, Amr and Kousik, Shreyas and Gao, Grace and Pavone, Marco and Johansson, Karl H},
  journal={IEEE Robotics and Automation Letters},
  volume={7},
  number={4},
  pages={10665--10672},
  year={2022},
  publisher={IEEE}
}

@article{chen2021fastrack,
  title={\href{https://doi.org/10.1109/TAC.2021.3059838}{Fastrack: a modular framework for real-time motion planning and guaranteed safe tracking}},
  author={Chen, Mo and Herbert, Sylvia L and Hu, Haimin and Pu, Ye and Fisac, Jaime Fernandez and Bansal, Somil and Han, SooJean and Tomlin, Claire J},
  journal={IEEE Transactions on Automatic Control},
  volume={66},
  number={12},
  pages={5861--5876},
  year={2021},
  publisher={IEEE}
}

@inproceedings{kousik2019safe,
  title={\href{https://doi.org/10.1115/DSCC2019-9214}{Safe, aggressive quadrotor flight via reachability-based trajectory design}},
  author={Kousik, Shreyas and Holmes, Patrick and Vasudevan, Ram},
  booktitle={Dynamic Systems and Control Conference},
  volume={59162},
  pages={V003T19A010},
  year={2019},
  organization={American Society of Mechanical Engineers}
}

@article{shao2021reachability,
  title={\href{https://doi.org/10.1109/LRA.2021.3063989}{Reachability-based trajectory safeguard (RTS): A safe and fast reinforcement learning safety layer for continuous control}},
  author={Shao, Yifei Simon and Chen, Chao and Kousik, Shreyas and Vasudevan, Ram},
  journal={IEEE Robotics and Automation Letters},
  volume={6},
  number={2},
  pages={3663--3670},
  year={2021},
  publisher={IEEE}
}

@article{kousik2020bridging,
  title={\href{https://doi.org/10.1177/0278364920943266}{Bridging the gap between safety and real-time performance in receding-horizon trajectory design for mobile robots}},
  author={Kousik, Shreyas and Vaskov, Sean and Bu, Fan and Johnson-Roberson, Matthew and Vasudevan, Ram},
  journal={The International Journal of Robotics Research},
  volume={39},
  number={12},
  pages={1419--1469},
  year={2020},
  publisher={SAGE Publications Sage UK: London, England}
}

@article{elbanhawi2014sampling,
  title={\href{https://doi.org/10.1109/ACCESS.2014.2302442}{Sampling-based robot motion planning: A review}},
  author={Elbanhawi, Mohamed and Simic, Milan},
  journal={Ieee access},
  volume={2},
  pages={56--77},
  year={2014},
  publisher={IEEE}
}

@inproceedings{tran2020nnv,
  title={\href{https://doi.org/10.1007/978-3-030-53288-8_1}{NNV: the neural network verification tool for deep neural networks and learning-enabled cyber-physical systems}},
  author={Tran, Hoang-Dung and Yang, Xiaodong and Manzanas Lopez, Diego and Musau, Patrick and Nguyen, Luan Viet and Xiang, Weiming and Bak, Stanley and Johnson, Taylor T},
  booktitle={International Conference on Computer Aided Verification},
  pages={3--17},
  year={2020},
  organization={Springer}
}

@inproceedings{tran2020verification,
  title={\href{https://doi.org/10.1007/978-3-030-53288-8_2}{Verification of deep convolutional neural networks using imagestars}},
  author={Tran, Hoang-Dung and Bak, Stanley and Xiang, Weiming and Johnson, Taylor T},
  booktitle={International conference on computer aided verification},
  pages={18--42},
  year={2020},
  organization={Springer}
}

@article{chung2021constrained,
  title={\href{https://arxiv.org/abs/2107.07696}{Constrained feedforward neural network training via reachability analysis}},
  author={Chung, Long Kiu and Dai, Adam and Knowles, Derek and Kousik, Shreyas and Gao, Grace X},
  journal={arXiv preprint arXiv:2107.07696},
  year={2021}
}

@phdthesis{althoff2010reachability,
  title={\href{https://mediatum.ub.tum.de/doc/1287517/document.pdf}{Reachability analysis and its application to the safety assessment of autonomous cars}},
  author={Althoff, Matthias},
  year={2010},
  school={Technische Universit{\"a}t M{\"u}nchen}
}

@article{everett2020robustness,
  title={\href{https://doi.org/10.1109/LCSYS.2020.3045323}{Robustness analysis of neural networks via efficient partitioning with applications in control systems}},
  author={Everett, Michael and Habibi, Golnaz and How, Jonathan P},
  journal={IEEE Control Systems Letters},
  volume={5},
  number={6},
  pages={2114--2119},
  year={2020},
  publisher={IEEE}
}

@inproceedings{achiam2017constrained,
  title={\href{https://proceedings.mlr.press/v70/achiam17a}{Constrained policy optimization}},
  author={Achiam, Joshua and Held, David and Tamar, Aviv and Abbeel, Pieter},
  booktitle={International conference on machine learning},
  pages={22--31},
  year={2017},
  organization={PMLR}
}

@article{schulman2017proximal,
  title={\href{https://arxiv.org/abs/1707.06347}{Proximal policy optimization algorithms}},
  author={Schulman, John and Wolski, Filip and Dhariwal, Prafulla and Radford, Alec and Klimov, Oleg},
  journal={arXiv preprint arXiv:1707.06347},
  year={2017}
}

@inproceedings{schulman2015trust,
  title={\href{https://proceedings.mlr.press/v37/schulman15.html}{Trust region policy optimization}},
  author={Schulman, John and Levine, Sergey and Abbeel, Pieter and Jordan, Michael and Moritz, Philipp},
  booktitle={International conference on machine learning},
  pages={1889--1897},
  year={2015},
  organization={PMLR}
}

@article{qin2022sablas,
  title={\href{https://doi.org/10.1109/LRA.2022.3142743}{Sablas: Learning safe control for black-box dynamical systems}},
  author={Qin, Zengyi and Sun, Dawei and Fan, Chuchu},
  journal={IEEE Robotics and Automation Letters},
  volume={7},
  number={2},
  pages={1928--1935},
  year={2022},
  publisher={IEEE}
}

@book{freeman2008robust,
  title={\href{https://books.google.com/books?id=ThHBE9xABUAC}{Robust nonlinear control design: state-space and Lyapunov techniques}},
  author={Freeman, Randy and Kokotovic, Petar V},
  year={2008},
  publisher={Springer Science \& Business Media}
}

@article{qin2021learning,
  title={\href{https://arxiv.org/abs/2101.05436}{Learning safe multi-agent control with decentralized neural barrier certificates}},
  author={Qin, Zengyi and Zhang, Kaiqing and Chen, Yuxiao and Chen, Jingkai and Fan, Chuchu},
  journal={arXiv preprint arXiv:2101.05436},
  year={2021}
}

@inproceedings{choi2021robust,
  title={\href{https://doi.org/10.1109/CDC45484.2021.9683085}{Robust control barrier--value functions for safety-critical control}},
  author={Choi, Jason J and Lee, Donggun and Sreenath, Koushil and Tomlin, Claire J and Herbert, Sylvia L},
  booktitle={2021 60th IEEE Conference on Decision and Control (CDC)},
  pages={6814--6821},
  year={2021},
  organization={IEEE}
}

@inproceedings{ames2019control,
  title={\href{https://doi.org/10.23919/ECC.2019.8796030}{Control barrier functions: Theory and applications}},
  author={Ames, Aaron D and Coogan, Samuel and Egerstedt, Magnus and Notomista, Gennaro and Sreenath, Koushil and Tabuada, Paulo},
  booktitle={2019 18th European control conference (ECC)},
  pages={3420--3431},
  year={2019},
  organization={IEEE}
}

@article{dai2021lyapunov,
  title={\href{https://arxiv.org/abs/2109.14152}{Lyapunov-stable neural-network control}},
  author={Dai, Hongkai and Landry, Benoit and Yang, Lujie and Pavone, Marco and Tedrake, Russ},
  journal={arXiv preprint arXiv:2109.14152},
  year={2021}
}

@article{zhang2023data,
  title={\href{https://doi.org/10.1109/TNNLS.2023.3339885}{Data-Driven Safe Policy Optimization for Black-Box Dynamical Systems With Temporal Logic Specifications}},
  author={Zhang, Chenlin and Lin, Shijun and Wang, Hao and Chen, Ziyang and Wang, Shaochen and Kan, Zhen},
  journal={IEEE Transactions on Neural Networks and Learning Systems},
  year={2023},
  publisher={IEEE}
}

@inproceedings{wu2020r3t,
  title={\href{https://doi.org/10.1109/ICRA40945.2020.9196802}{R3T: Rapidly-exploring random reachable set tree for optimal kinodynamic planning of nonlinear hybrid systems}},
  author={Wu, Albert and Sadraddini, Sadra and Tedrake, Russ},
  booktitle={2020 IEEE International Conference on Robotics and Automation (ICRA)},
  pages={4245--4251},
  year={2020},
  organization={IEEE}
}

@article{hornik1989multilayer,
  title={\href{https://doi.org/10.1016/0893-6080(89)90020-8}{Multilayer feedforward networks are universal approximators}},
  author={Hornik, Kurt and Stinchcombe, Maxwell and White, Halbert},
  journal={Neural networks},
  volume={2},
  number={5},
  pages={359--366},
  year={1989},
  publisher={Elsevier}
}

@inproceedings{althoff2015introduction,
  title={\href{https://mediatum.ub.tum.de/doc/1280439/document.pdf}{An introduction to CORA 2015}},
  author={Althoff, Matthias},
  booktitle={Proc. of the workshop on applied verification for continuous and hybrid systems},
  pages={120--151},
  year={2015}
}

@article{montufar2014number,
  title={\href{https://proceedings.neurips.cc/paper/2014/hash/109d2dd3608f669ca17920c511c2a41e-Abstract.html}{On the number of linear regions of deep neural networks}},
  author={Montufar, Guido F and Pascanu, Razvan and Cho, Kyunghyun and Bengio, Yoshua},
  journal={Advances in neural information processing systems},
  volume={27},
  year={2014}
}

@article{hanin2019universal,
  title={\href{https://www.mdpi.com/2227-7390/7/10/992}{Universal function approximation by deep neural nets with bounded width and relu activations}},
  author={Hanin, Boris},
  journal={Mathematics},
  volume={7},
  number={10},
  pages={992},
  year={2019},
  publisher={MDPI}
}

@article{arora2016understanding,
  title={\href{https://arxiv.org/abs/1611.01491}{Understanding deep neural networks with rectified linear units}},
  author={Arora, Raman and Basu, Amitabh and Mianjy, Poorya and Mukherjee, Anirbit},
  journal={arXiv preprint arXiv:1611.01491},
  year={2016}
}

@article{batista2024deep,
  title={\href{https://arxiv.org/abs/2407.08263}{A Deep Reinforcement Learning Framework and Methodology for Reducing the Sim-to-Real Gap in ASV Navigation}},
  author={Batista, Luis FW and Ro, Junghwan and Richard, Antoine and Schroepfer, Pete and Hutchinson, Seth and Pradalier, Cedric},
  journal={arXiv preprint arXiv:2407.08263},
  year={2024}
}

@article{makoviychuk2021isaac,
  title={\href{https://arxiv.org/abs/2108.10470}{Isaac gym: High performance gpu-based physics simulation for robot learning}},
  author={Makoviychuk, Viktor and Wawrzyniak, Lukasz and Guo, Yunrong and Lu, Michelle and Storey, Kier and Macklin, Miles and Hoeller, David and Rudin, Nikita and Allshire, Arthur and Handa, Ankur and others},
  journal={arXiv preprint arXiv:2108.10470},
  year={2021}
}

@article{kousik2019technical,
  title={\href{https://arxiv.org/abs/1904.05728}{Technical report: Safe, aggressive quadrotor flight via reachability-based trajectory design}},
  author={Kousik, Shreyas and Holmes, Patrick and Vasudevan, Ramanarayan},
  journal={arXiv preprint arXiv:1904.05728},
  year={2019}
}

@article{orthey2024multilevel,
  title={\href{https://doi.org/10.1177/02783649231209337}{Multilevel motion planning: A fiber bundle formulation}},
  author={Orthey, Andreas and Akbar, Sohaib and Toussaint, Marc},
  journal={The international journal of robotics research},
  volume={43},
  number={1},
  pages={3--33},
  year={2024},
  publisher={SAGE Publications Sage UK: London, England}
}

@article{liu2023simultaneous,
  title={\href{https://doi.org/10.1109/TCST.2023.3283446}{Simultaneous planning and execution for quadrotors flying through a narrow gap under disturbance}},
  author={Liu, Zhou and Cai, Lilong},
  journal={IEEE Transactions on Control Systems Technology},
  volume={31},
  number={6},
  pages={2644--2659},
  year={2023},
  publisher={IEEE}
}

@article{yu2023gaussian,
  title={\href{https://doi.org/10.1109/LRA.2023.3256134}{A Gaussian variational inference approach to motion planning}},
  author={Yu, Hongzhe and Chen, Yongxin},
  journal={IEEE Robotics and Automation Letters},
  volume={8},
  number={5},
  pages={2518--2525},
  year={2023},
  publisher={IEEE}
}

@inproceedings{berenson2009manipulation,
  title={\href{https://doi.org/10.1109/ROBOT.2009.5152399}{Manipulation planning on constraint manifolds}},
  author={Berenson, Dmitry and Srinivasa, Siddhartha S and Ferguson, Dave and Kuffner, James J},
  booktitle={2009 IEEE international conference on robotics and automation},
  pages={625--632},
  year={2009},
  organization={IEEE}
}

@article{liu2021algorithms,
  title={Algorithms for verifying deep neural networks},
  author={Liu, Changliu and Arnon, Tomer and Lazarus, Christopher and Strong, Christopher and Barrett, Clark and Kochenderfer, Mykel J and others},
  journal={Foundations and Trends{\textregistered} in Optimization},
  volume={4},
  number={3-4},
  pages={244--404},
  year={2021},
  publisher={Now Publishers, Inc.}
}

@article{o2019f1,
  title={\href{https://arxiv.org/abs/1901.08567}{F1/10: An open-source autonomous cyber-physical platform}},
  author={O'Kelly, Matthew and Sukhil, Varundev and Abbas, Houssam and Harkins, Jack and Kao, Chris and Pant, Yash Vardhan and Mangharam, Rahul and Agarwal, Dipshil and Behl, Madhur and Burgio, Paolo and others},
  journal={arXiv preprint arXiv:1901.08567},
  year={2019}
}

@article{forets2021lazysets,
  title={\href{https://arxiv.org/abs/2110.01711}{LazySets. jl: Scalable symbolic-numeric set computations}},
  author={Forets, Marcelo and Schilling, Christian},
  journal={arXiv preprint arXiv:2110.01711},
  year={2021}
}

@article{mamakoukas2020memory,
  title={\href{https://proceedings.neurips.cc/paper/2020/hash/9cd78264cf2cd821ba651485c111a29a-Abstract.html}{Memory-efficient learning of stable linear dynamical systems for prediction and control}},
  author={Mamakoukas, Giorgos and Xherija, Orest and Murphey, Todd},
  journal={Advances in Neural Information Processing Systems},
  volume={33},
  pages={13527--13538},
  year={2020}
}

@article{miller2015ergodic,
  title={\href{https://doi.org/10.1109/TRO.2015.2500441}{Ergodic exploration of distributed information}},
  author={Miller, Lauren M and Silverman, Yonatan and MacIver, Malcolm A and Murphey, Todd D},
  journal={IEEE Transactions on Robotics},
  volume={32},
  number={1},
  pages={36--52},
  year={2015},
  publisher={IEEE}
}
\appendix
\subsection{H-Polytope Operations}
Consider a pair of H-polytopes $\Poly_1 = \hpoly(\Acon_1, \bcon_1)$ and $\Poly_2 = \hpoly(\Acon_2, \bcon_2)$.
Their intersection $\cap$ is an H-polytope:
\begin{align}
\begin{split}\label{eq:intersection}
    \Poly_1 \cap \Poly_2 &= \left\{\xv\ |\ \Acon_1\xv\leq\bcon_1, \Acon_2\xv\leq\bcon_2\right\},\\
    &= \hpoly\left(\begin{bmatrix}
        \Acon_1 \\
        \Acon_2
    \end{bmatrix}, \begin{bmatrix}
        \bcon_1 \\
        \bcon_2
    \end{bmatrix}\right).
\end{split}
\end{align}

Their Cartesian product $\times$ is also an H-polytope \cite{herceg2013multi}:
\begin{subequations}
\begin{align}
    \Poly_1 \times \Poly_2 &= \left\{\begin{bmatrix}
        \xv\\\yv
    \end{bmatrix}\ \Bigr|\ \xv \in \Poly_1, \yv \in \Poly_2\right\},\\
    &= \hpoly\left(\begin{bmatrix}
        \Acon_1 & \zeros\\
        \zeros & \Acon_2
    \end{bmatrix}, \begin{bmatrix}
        \bcon_1\\\bcon_2
    \end{bmatrix}\right).
\end{align}
\end{subequations}

The Minkowski sum $\oplus$ is defined as \cite{herceg2013multi}:
\begin{align}\label{eq:reg_minkowski_sum}
    \Poly_1 \oplus \Poly_2 &= \left\{\xv + \yv\ |\ \xv \in \Poly_1, \yv \in \Poly_2\right\}.
\end{align}

Similarly, the Pontryagin Difference $\ominus$ is defined as \cite{herceg2013multi}:
\begin{align}
    \Poly_1 \ominus \Poly_2 &= \left\{\xv \in \Poly_1\ |\ \xv + \yv \in \Poly_1\ \forall\ \yv \in \Poly_2\right\}.
\end{align}

\subsection{AH-Polytope Operations}

Consider a pair of AH-polytopes $\Qoly_1 = \ahpoly(\Acon_1, \bcon_1, \Ccon_1, \dcon_1)$, $\Qoly_2 = \ahpoly(\Acon_1, \bcon_1, \Ccon_1, \dcon_1)$.
The intersection $\cap$ of two AH-polytopes is an AH-polytope \cite{sadraddini2019linear}:
\begin{subequations}
\begin{align}
    \Qoly_1 \cap \Qoly_2 &=\left\{\xv\ |\ \xv \in \Qoly_1, \xv \in \Qoly_2 \right\},\\
    &=\ahpoly\left(\begin{bmatrix}
        \Acon_1 & \zeros\\
        \zeros & \Acon_2\\
        \Ccon_1 & -\Ccon_2\\
        -\Ccon_1 & \Ccon_2
    \end{bmatrix}, \begin{bmatrix}
        \bcon_1\\
        \bcon_2\\
        \dcon_2-\dcon_1\\
        \dcon_1-\dcon_2
    \end{bmatrix}, [\Ccon_1, \zeros], \dcon_1\right).
\end{align}
\end{subequations}

The convex hull of two AH-polytopes is an AH-polytope \cite{sadraddini2019linear}:
\begin{subequations}
\begin{align}
    &\conv{\Qoly_1, \Qoly_2} \notag\\
    =&\left\{\xv + \gams (\yv - \xv)\ |\ 0 \leq \gams \leq 1, \xv, \yv \in \Qoly_1 \cup \Qoly_2\right\},\\
    =&\ahpoly\left(\begin{bmatrix}
        \Acon_1 & \zeros & -\bcon_1\\
        \zeros & \Acon_2 & \bcon_2\\
        \zeros & \zeros & 1\\
        \zeros & \zeros & -1
    \end{bmatrix}, \begin{bmatrix}
        \zeros\\
        \bcon_2\\
        1\\
        0
    \end{bmatrix}, [\Ccon_1, \Ccon_2, \dcon_1-\dcon_2], \dcon_2\right).
\end{align}
\end{subequations}

Finally, the projection of an $\ndim$-dimensional H-polytope $\hpoly(\Acon, \bcon)$ onto its first $\mdim$ dimensions, $\mdim \leq \ndim$, is an $\mdim$-dimensional AH-polytope:
\begin{subequations}
\begin{align}
    \proj[\mdim]{\hpoly(\Acon, \bcon)} &= \left\{
        \begin{bmatrix}
        \eye_{\mdim}, \ 
        \zeros
    \end{bmatrix}\xv\ |\
    \xv\in\hpoly(\Acon, \bcon)\right\},\\
    &= \ahpoly\left(\Acon, \bcon, \begin{bmatrix}
        \eye_{\mdim}, \ 
        \zeros
    \end{bmatrix}, \zeros\right).
\end{align}
\end{subequations}

\subsection{Detailed Proof of Theorem 6}

By Lemma \ref{lem:traj_to_real}, it is sufficient to show that, if there is a point $[{\planv_0}\tp, \paramv\tp]\tp$ in $\Rbrs$ that collides with $\tilde{\Obs}_{\ts, \js}$ between time $\ts$ and $\ts + \Delta\ts$, then $\planv_0$ must be in the convex hull between $\Obsbrs_{\ts, \js}$ and $\Obsbrs_{\ts+\Delta\ts, \js}$.
    Mathematically, if $\exists [{\planv_0}\tp, \paramv\tp]\tp \in \Rbrs, \alpha \in [0, 1]$ such that 
    \begin{align}
        \hat{\planv}(\ts) + \alpha(\hat{\planv}(\ts + \Delta\ts) - \hat{\planv}(\ts)) \in \tilde{\Obs}_{\ts, \js},
    \end{align}
    then $\exists \beta\in[0, 1], \planv_1\in\Obsbrs_{\ts, \js}, \planv_2\in\Obsbrs_{\ts+\Delta\ts, \js}$ such that
    \begin{align}
        \planv_0 &= \planv_1 + \beta(\planv_2 - \planv_1).
    \end{align}

    From Assumption \ref{ass:eti} and \eqref{eq:pwa_map}, we have
    \begin{subequations}
    \begin{align}
        \hat{\planv}(\ts) =& \planv_0 + (\Ccon_{\sap_\is, \ts})_{1:\nplan, (\nplan+1):\nparam}\paramv + (\dcon_{\sap_\is, \ts})_{(\nplan+1):\nparam},\notag\\
        :=& \planv_0 + \Ccon_\ts \paramv + \dcon_\ts,\\
        \hat{\planv}(\ts + \Delta\ts) =& \planv_0 + (\Ccon_{\sap_\is, \ts + \Delta\ts})_{1:\nplan, (\nplan+1):\nparam}\paramv + (\dcon_{\sap_\is, \ts + \Delta\ts})_{(\nplan+1):\nparam},\notag\\
        :=& \planv_0 + \Ccon_{\ts + \Delta\ts} \paramv + \dcon_{\ts + \Delta\ts},
    \end{align}
    \end{subequations}
    for all $[{\planv_0}\tp, \paramv\tp]\tp \in \Rbrs$.

    Thus, from translation invariance, we can always find some $\planv_1, \planv_2$ such that $\planv_1 + \Ccon_\ts \paramv + \dcon_\ts \in \tilde{\Obs}_{\ts, \js}$ and $\planv_2 + \Ccon_{\ts + \Delta\ts} \paramv + \dcon_{\ts + \Delta\ts} \in \tilde{\Obs}_{\ts, \js}$, which implies $\planv_1\in\Obsbrs_{\ts, \js}$ and $\planv_2\in\Obsbrs_{\ts+\Delta\ts, \js}$.
    We also have
    \begin{align}
        &\hat{\planv}(\ts) + \alpha(\hat{\planv}(\ts + \Delta\ts) - \hat{\planv}(\ts))\notag\\
        =& \planv_0 + \Ccon_\ts \paramv + \dcon_\ts + \alpha((\Ccon_{\ts + \Delta\ts} - \Ccon_{\ts})\paramv + \dcon_{\ts + \Delta\ts} - \dcon_{\ts}),\notag\\
        :=& \planv_0 + \Ccon_\ts \paramv + \dcon_\ts + \alpha\hat{\planv}_d.
    \end{align}

    Let $\planv_1 + \Ccon_\ts \paramv + \dcon_\ts = \planv_2 + \Ccon_{\ts + \Delta\ts} \paramv + \dcon_{\ts + \Delta\ts} = \planv_0 + \Ccon_\ts \paramv + \dcon_\ts + \alpha\hat{\planv}_d$.
    Then $\planv_1 = \planv_0 + \alpha\hat{\planv}_d$, $\planv_2 = \planv_0 + (\alpha-1)\hat{\planv}_d$.
    Therefore
    \begin{align}
        \planv_1 + \beta(\planv_2 - \planv_1) &= \planv_0 + \alpha\hat{\planv}_d - \beta\hat{\planv}_d,
    \end{align}
    which is equal to $\planv_0$ when $\beta = \alpha \in [0, 1]$.

\end{document}